\documentclass[reqno,oneside,letterpaper,10pt]{article}
\usepackage[utf8]{inputenc} 
\pdfoutput=1        
\usepackage{fullpage}
\usepackage[textwidth=146mm,textheight=214mm]{geometry} %

\usepackage[utf8]{inputenc} 
\usepackage{fullpage}

\usepackage{url} 
\usepackage{bm}
\usepackage{amsmath,amssymb}
\usepackage{xspace}
\usepackage{subfigure, graphicx}
\usepackage[colon]{natbib}
\usepackage{color}

\usepackage{algorithm}
\usepackage{algorithmic}

\usepackage[colorlinks,citecolor=blue,urlcolor=blue,linkcolor=blue]{hyperref}

\newcommand{\Holder}{H\"older\xspace}

\renewcommand{\t}{\mathbf{t}}
\newcommand{\Zspace}{\mathcal{Z}}

\newcommand{\f}{f}
\newcommand{\g}{g}

\newcommand{\q}{q}
\newcommand{\I}{I}

\newcommand{\entropy}{\mathcal{H}}

\newcommand{\trace}{\mathop{\textrm{tr}}}

\newcommand{\diag}[1]{\mathop{\textrm{diag}}\left(#1\right)}

\newcommand{\tauspace}{\mathcal{T}}

\newcommand{\truncnormdist}{\mathcal{TN}}

\newcommand{\basem}{\nu}

\newcommand{\transp}{^{T}}

\newcommand{\E}[2]{\mathbb{E}_{#1}\left[#2\right]}
\newcommand{\unigint}[1]{{U}_{\left[#1\right]}}
\newcommand{\Expectation}[2]{\int{#2}d\nu(\Z)}

\def\Indic#1{\mathbb{I}_{\{#1\}}}

\def\A{A}

\def\btau{{\bm{\tau}}}

\def\Z{Z}
\def\z{z}
\def\p{p}

\def\proba{p}

\newcommand{\logit}{\mathsf{logit}}
\renewcommand{\Re}{\mathbb{R}}

\newenvironment{proof}{\textsc{proof.}\it}{\hfill{$\Box$}}

\newtheorem{lemma}{Lemma}
\newtheorem{theorem}{Theorem}
\newtheorem{proposition}{Proposition}

\begin{document}

\title{Approximate Inference with the Variational \Holder Bound}
\date{}

\author{ Guillaume Bouchard \\
Department of Computer Science\\ University College London\\
\href{guillaume.m.bouchard@gmail.com}{\nolinkurl{guillaume.m.bouchard@gmail.com}} \\
\and
Balaji Lakshminarayanan  \\
Gatsby Unit\\ University College London          \\
\href{balaji@gatsby.ucl.ac.uk}{\nolinkurl{balaji@gatsby.ucl.ac.uk}}
}

\maketitle

\begin{abstract} 
We introduce the Variational \Holder (VH) bound as an alternative to Variational Bayes (VB) for approximate Bayesian inference. Unlike VB which typically involves maximization of a non-convex lower bound with respect to the variational parameters, the VH bound involves minimization of a convex upper bound to the intractable integral with respect to the variational parameters. Minimization of the VH bound is a convex optimization problem; hence the VH method can be applied using off-the-shelf convex optimization algorithms and the approximation error of the VH bound can also be analyzed using tools from convex optimization literature. 
We present experiments on the task of integrating a truncated multivariate Gaussian distribution and compare our method to VB, EP and a state-of-the-art numerical integration method for this problem. 
\end{abstract}

\section{Introduction}
Many Bayesian machine learning problems involve an intractable sum or integral, for which
numerical approximations methods have been derived. Approximate Bayesian inference techniques can be broadly classified into sampling-based (e.g.~Markov chain Monte Carlo) and optimization-based (e.g.~variational Bayes, expectation propagation) methods.
While sampling techniques
are widely used to explore the space and compute the %
statistics  of
interest for the problem, they are not always satisfying due to their stochastic
nature and it is hard to assess convergence. %

Many algorithms involve the computation an objective function, 
such as a loss function, a negative log-likelihood or a energy
criterion. However, the objective function itself often includes sums that are
slow to compute, requiring the approximation of this sum. This is the case in
empirical Bayes method (a.k.a. type-II maximum likelihood), mixture models with
a latent state space such as high-order hidden Markov models and restricted
Boltzmann machines, or even a simple Maximum Likelihood (ML) with fully observed
data:
the ML estimator of exponential family models with non-standard feature
functions requires the computation of the partition-function, which is
intractable as soon as the feature functions or the parameter space do not
belong the restricted class of \emph{tractable} models, including Gaussian
distributions and tree-structure graphical models for non-Gaussian
distributions. For other models, the partition function needs to be approximated
and a full set of approximate inference algorithms have been designed during the
last decades in including pseudo-likelihood
approaches~\citep{gourieroux1984pseudo}, but these approaches do not show 
good empirical performances and do not really help to predict the likelihood of the 
observations.
For other approximation schemes based on mean field approximations, obtaining
algorithms with provable polynomial-time convergence guarantees and other
theoretical guarantees is hard in general~\citep{wainwright2008graphical}.
   
In Bayesian statistics, many deterministic inference approaches have been
proposed, the main ones being Variational Bayes (VB)~\citep{williams1991mean,jordan1999introduction,attias2000variational},
Expectation-Propagation (EP)~\citep{minka2005divergence}, and Tree-Reweighted sum-product (TRW)~\citep{wainwright2005new}.
For continuous variables, classical approximate inference schemes are based on
EP or the Variational Gaussian (VG) representation, which is basically the information
inequality applied to the Gaussian case~\citep{ChallisBarber2011}.
However, the VG bound is known to be a crude inequality which tends to under
estimate the variance, leading to poor results in situations where variance
estimates are crucial, for example in Bayesian experimental
design~\citep{SeegerNickisch2011}.
More interestingly, \citet{liu11d} showed that new inference algorithms can be obtained 
by minimizing the generalized \Holder's inequality applied on the partition function of a discrete graphical model.
Such algorithms do not suffer from the zero-avoiding behavior of VB
and the lack of convergence guarantees of EP, and has strong connections with the TRW convex upper bound to the partition function. 

In this work, we introduce the Variational \Holder (VH) inequality, a family of tractable upper bounds to the product
of potentials, possibly defined on a continuous space, unlike previous work focusing %
only on the discrete case. Hence, our bound generalizes earlier work by \citet{liu11d} and is simpler in construction.  
{We show that we can infer continuous latent variables values in a Bayesian inference problem where the unnormalized integral is a product of two potentials corresponding to the prior and likelihood respectively. The optimization with respect to the variational parameters in VH is a convex optimization problem and can be solved using off-the-shelf tools. We compare the performance of our method to VB, EP and a state-of-the-art numerical optimizer on the task of integrating a truncated multivariate Gaussian distribution.}

\section{Variational \Holder bound}
\label{sec:vh}

\paragraph{Notations}

We define a probabilty space $(\Omega,\mathcal{F}, \nu)$ where $\Omega$ 
is a sample space and $\mathcal{F}$ a sigma-algebra defined on it.
Let $\Z$ be a $\nu$-distributed random variable taking values in a Hilbert space $\Zspace$.
We make use of $\mathcal{L}_p$ norms $\|.\|_p$ repeatedly, where $p>1$ and
$\|\f\|_p:=\bigl(\Expectation{\nu}{|f(\Z)|^p}\bigr)^{\frac 1p}$ for $p<\infty$ and $\|\f\|_\infty:=\sup_z |f(z)|$.
Let: %
\begin{eqnarray}
I^*:=\Expectation{}{\gamma_1(\Z)\gamma_2(\Z)}=\|\gamma_1\gamma_2\|_1
\end{eqnarray}
be the integral we want to approximate, also 
called the partition function of the unnormalized distribution with density $\gamma(\Z) :=\gamma_1(\Z)\gamma_2(\Z)$.

\newcommand{\balpha}{{\bm{\alpha}}}

\paragraph{Upper bound to the partition function}
We define the following functional:
\begin{eqnarray}
\label{eq:VH}
\bar I_\balpha(\Psi)&:=
		\|\gamma_1\Psi\|_{\alpha_1}\|\gamma_2/\Psi\|_{\alpha_2}
\end{eqnarray}
where the argument of $\bar I_\balpha$ is a positive function $\Psi:\Zspace\mapsto \Re^+$ %
which we refer to as the \emph{pivot function}
and $\balpha=(\alpha_1,\alpha_2)\in\Re^{+2}$. 
The main result of this paper is the study of a new inequality to the log-partition function, that we call \emph{Variational \Holder}~(VH) inequality because it
corresponds to a direct application of the well-known \Holder's inequality:
\begin{theorem}
\label{th:VH}
Let $\gamma_1$ and $\gamma_2$ be two positive measures defined on $\Zspace$. 
The following inequality:
\begin{align}
\label{eq:VH-logspace}
&I^* \le \bar I_\balpha(\Psi)
\enspace,
\end{align}
holds for any positive scalars $\alpha_1$ and $\alpha_2$ such that $\frac{1}{\alpha_1} + \frac{1}{\alpha_2}=1$ and any function $\Psi:\Zspace\mapsto \Re^+$.
Equality holds if for almost all $\z\in\Zspace$, 
$\Psi(\z) = \gamma_1(\z)^{-\frac {1}{\alpha_2}}  \gamma_2(\z)^{\frac{1} {\alpha_1}}$.
\end{theorem}
\begin{proof}
The bound~\eqref{th:VH} is obtained using \Holder's inequality $\|fg\|_1\le\|f\|_{\alpha_1}\|g\|_{\alpha_2}$ for $f\leftarrow \gamma_1\Psi$ and $g\leftarrow \gamma_2/\Psi$.
The tightness result is given by a direct calculation:
{\small
$ \bar I_\balpha( \gamma_1^{-\frac {1}{\alpha_2}}  \gamma_2^{\frac{1} {\alpha_1}}) = \|\gamma_1^{1-\frac {1}{\alpha_2}} \gamma_2^{\frac{1} {\alpha_1}}\|_{\alpha_1}\|\gamma_1(\z)^{\frac{1} {\alpha_2}}\gamma_2^{1-\frac {1}{\alpha_1}} \|_{\alpha_2}=\|(\gamma_1\gamma_2)^{\frac {1}{\alpha_1}}\|_{\alpha_1}\|(\gamma_1\gamma_2)^{\frac {1}{\alpha_2}}\|_{\alpha_2}
=\|\gamma_1\gamma_2\|_1^{\frac {1}{\alpha_1}} \|\gamma_1\gamma_2\|_1^{\frac {1}{\alpha_2}} =  \|\gamma_1\gamma_2\|_1=I^*$.
}
\end{proof}

The key insight in the VH bound over the standard \Holder bound 
is that we will choose the pivot function $\Psi$ so that the bound is as close as possible to the target integral.
The upper bound on the right-hand side has several useful properties. The first one is that the upper bound 
can be \emph{tractable} even if the original quantity $\|\gamma_1\gamma_2\|_1$ is intractable. 
The second useful property of the log of the bound is \emph{convex} in $\log(\Psi)$, which makes it 
convenient to optimize, and in particular using gradient descent methods that are
provably convergent in polynomial time. Finally, this bound has theoretical properties that make it 
suitable for approximating distributions, as shown in the next section.

\section{Theoretical Guarantees}
In this section, we show that under mild conditions, the VH bound is good for variational inference; 
when the upper bound is close to the target partition function, then the resulting approximation
is also close to the target distribution $\proba^* :=\frac{\gamma_1\gamma_2}{I^*} $.

\begin{proposition}  
\label{th:approx}
For any $\varepsilon>0$, the inequality $I^* > (1-\varepsilon) \bar I_\balpha(\Psi)$ implies that:
\begin{align}
\left\|
	\frac{(\gamma_1\Psi)^{\alpha_1}}{\|\gamma_1\Psi\|^{\alpha_1}_{\alpha_1}} 
	- 
	\proba^*
\right\|_2 < \sqrt{2\varepsilon} + \varepsilon
\quad &  \mathrm{if} \quad \alpha_1 \le 2, \mathrm{\ and}
\\
\left\|
	\frac{(\gamma_2/\Psi)^{\alpha_2}}{\|\gamma_2/\Psi\|^{\alpha_2}_{\alpha_2}} 
	- 
	\proba^*
\right\|_2 < \sqrt{2\varepsilon} + \varepsilon
\quad  & \mathrm{if} \quad \alpha_1 \ge 2
\enspace.
\end{align}
\end{proposition} 
The proof is given in the appendix for clarity, but it is novel and is one of the key contributions of the paper. 
Proposition~\ref{th:approx} shows that the smaller the relative gap of the VH inequality is, the better the functions $(\gamma_1 \Psi)^{\alpha_1}$ or $(\gamma_2/\Psi)^{\alpha_2}$ can approximate the target distribution $p^*$. This approximation is useful when $\gamma_1\gamma_2$ is hard to integrate, but $(\gamma_1 \Psi)^{\alpha_1}$ and $(\gamma_2/\Psi)^{\alpha_2}$ are easy to integrate.

Now that the approximation properties of the VH bounds have been highlighted, we describe how to effectively use these results in practice.

\section{\Holder Variational Bayes}
Based on the previous results, we obtain a variational algorithm to approximate product of factors by tractable factors.
To do that, we choose the pivot function $\Psi$ in a properly chosen \emph{tractable family} $\mathcal{F}=\{\Psi(.;\tau),\tau\in\tauspace\}$
 where $\tauspace$ is the set of variational parameters that defines the family.
Then, we obtain estimates for $\alpha_1$ and $\tau$ by minimizing\footnote{We optimize \eqref{eq:varobj} with respect to $\logit(1/\alpha_1)$ instead of $\alpha_1$, where $\logit(u)=\log(\frac{u}{1-u})$, since the former is an unconstrained minimization whereas the latter is a constrained minimization problem.}
 the VH bound~\eqref{eq:VH} over $\mathcal{F}$:
\begin{align}\label{eq:varobj}
  (\hat\tau,\hat\alpha_1) \in \arg\min_{\tauspace\otimes\Re} \bar I_\balpha(\Psi(.,;\tau))
\end{align}

Once the optimized values $(\hat\tau,\hat\alpha_1)$ have been found, the approximation to the exact intractable distribution $\proba^*$ are given by:
\begin{eqnarray}
\hat\proba_1(Z) := \frac{(\gamma_1(Z)\Psi(Z;\hat\btau))^{\alpha_1}}{\|\gamma_1\Psi(.;\hat\btau)\|^{\alpha_1}_{\alpha_1}}
\end{eqnarray}
and 
\begin{eqnarray}
\hat\proba_2(Z) := \frac{(\gamma_2(Z)/\Psi(Z;\hat\btau))^{\alpha_2}}{\|\gamma_2/\Psi(.;\hat\btau)\|^{\alpha_2}_{\alpha_2}}
\end{eqnarray}
Other moments can be computed in a similar fashion. 
Note that by choosing the proper approximating family $\mathcal{F}$, both distributions are assumed to be tractable, i.e. we can compute their normalization constant efficiently.
According to Proposition~\ref{th:approx}, one should choose $\hat\proba_1$ or $\hat\proba_2$ depending whether $\hat\alpha_1$ is 
smaller or greater than 2, but we also considered a convex combination of these two tractable distributions. This amounts to using the 
following mixture model:
\begin{eqnarray}
\hat\proba_{12}(Z) := \frac {1}{\alpha_1}\proba_1(Z) + \frac {1}{\alpha_2}\proba_2(Z)
\enspace
\end{eqnarray}
as an approximating distribution.
In their seminal paper, \citet{liu11d} also minimized the \Holder's bound with respect to parameters and exponent values, but
their approach is restricted to %
 discrete graphical models, and they applied this idea in the framework of the bucket elimination algorithm.

So far, the VH bound is very general. It can be applied on discrete and continuous spaces. The sole assumption we made is that $\proba^*\in\mathcal{L}_1$, $\gamma_1\Psi(:,\tau)\in\mathcal{L}_{\alpha_1}$ and$\gamma_2\/\Psi(:,\tau)\in\mathcal{L}_{\alpha_2}$ for any $\tau\in\tauspace$ and any $\alpha_k$ in the
range of \Holder's exponents that are considered. We now turn on to a specific class of functions to illustrate how the VH bound is used in practice.

\section{Application: Gaussian Integration}

\subsection{Problem Definition}

Using the notations of the previous section, we define $\gamma_1(t)=\prod_{i=1}^n f_i(t_i)$ where each function $f_i:\Re\mapsto \Re$ is univariate. We also define $\gamma_2(t)$ with $e^{-\frac 12 t\transp \A t + b\transp t}$, where $A$ is a symmetric $n\times n$ matrix and $\bm b\in\Re^n$.
We use the Lebesgue measure for $\nu$. Assume we want to evaluate: 
\begin{eqnarray} 
I^*:=\int_{\Re^n}  \prod_{i=1}^n f_i(t_i)
e^{-\frac 12 \t\transp \A \t + \bm b\transp \t} d\t
\enspace.
\label{univ-gaussian-integral} 
\end{eqnarray}
This type of integral is common in machine learning~\citep{seeger2010gaussian}. Typically, this corresponds to the marginal data probability --- a.k.a. the evidence --- of a linear regression model with known variance and sparse priors, $n$ being the number of variables. 
Up to an affine change of variable to obtain orthogonal univariate factors, this integral corresponds also to the data evidence 
of in a generalized linear model with Gaussian prior, where $n$ is the number of independent observations. 
The functions $\gamma_1$ and $\gamma_2$ alone are easy to integrate. This remain true when they are multiplied by a
Gaussian potential with diagonal covariance matrices, so that we can choose the following variational family
$\{\Psi(\t;\btau) ,\btau\in\Re^{2n}\}$:
\begin{eqnarray}
	\left\lbrace
		e^{-\frac 12 \t\transp\diag{\btau_1}\t +  \btau_2\transp\t},
		\quad \btau_1\in\Re^n,\btau_2 \in \Re^n 
	\right\rbrace
.
\label{Gaussian-pivot-functions}
\end{eqnarray}
A approximation to the integral~(\ref{univ-gaussian-integral})  is obtained by minimizing the upper bound $\bar\I_\balpha$ given in Equation~(\ref{eq:VH}).

\subsection{Integration of orthogonal univariate function}
\label{sec:univ-integral}
The first term in the bound $\|\gamma_1\Psi(.,\btau)\|_{\alpha_1}$ can be obtained 
efficiently in terms of univariate integrals:
\begin{eqnarray} 
\|\gamma_1\Psi(.,\btau)\|_{\alpha_1}^{\alpha_1}
&=&
\prod_{i=1}^n
\int_\Re f_i^{\alpha_1}(t_i)
e^{\alpha_1(-\frac{\tau_{1i} t_i^2}2 + \tau_{2i} t_i)}
dt_i
\nonumber\\
&=&
\prod_{i=1}^n
\unigint{f_i}\left(\tau_{1i},\tau_{2i},\alpha_1\right)
\enspace,
\label{ortho-univ}
\end{eqnarray}
where the univariate integrals $\unigint{h}:\Re^2\times [0;1]\mapsto\Re$ are defined as

\begin{eqnarray} 
\unigint{h}(a,b,\alpha_1):=\int_\Re \left(h(t)e^{-\frac 12 a t^2 + b t}\right)^{\alpha_1} dt_i
\enspace.
\end{eqnarray} 
Here, $h$ is an arbitrary univariate function $\Re\mapsto\Re^+$. These integrals
can be efficiently computed using quadrature integration (e.g. recursive
adaptive Simpson quadrature), but in many practical applications, the same
functions $f_i$ in Equation~(\ref{ortho-univ}) are used for many factors. A considerable speedup can be obtained
by designing integrals dedicated to some functions (in practice, using pre-computed functions with linear interpolation 
 is 10 to 100 times faster than running a new quadrature every time). One important special is the step function:
$f_i(x)=\Indic{x\ge 0}$ for all $i\in\{1,\cdots,n\}$.
For this function and using Gaussian pivot functions as specified in Equation~(\ref{Gaussian-pivot-functions}), we obtain a closed form expression in terms of normal CDF function~$\Phi$:			
\begin{eqnarray}
		\unigint{\Indic{\cdot\ge0}}(a,b,\alpha) 
		&		=		&
		\sqrt{\frac{2\pi}{\alpha a}} 
		  \Phi\left(b \sqrt{\frac{\alpha}{a}}\right)
e^{\frac{\alpha b^2}{2 a}}
\enspace.
\end{eqnarray}	
If there is no truncation in some dimensions, the constant one function gives:
$\unigint{\bm{1}}(a,b,\alpha) =\sqrt{\frac{2\pi}{\alpha a}} e^{\frac{\alpha b^2}{2 a}}.$

\subsection{Gaussian Integration}
Concerning the other factor $\|\gamma_2/\Psi(.,\btau)\|_{\alpha_2}$, its log-quadratic form corresponds to a standard Gaussian integral: 
\begin{flalign} 
\|\gamma_2/\Psi(.,\btau)\|_{\alpha_2}^{\alpha_2}
=
\int_{\Re^n} 
e^{-\frac{\alpha_2}{2} \t\transp (\A - \diag{\btau_1}) \t + {\alpha_2}(\bm b - \btau_2)\transp \t} d\t
= {(2\pi)^{\frac n2}}  e^{J\left(
		{\alpha_2}(A- \diag{\btau_1}),
		{\alpha_2}(\bm b - \btau_2)
	\right)},	
\nonumber
\end{flalign}
where $J(M,v):=-\frac 12 \log |M| + \frac 12 v\transp M^{-1} v$.

\begin{figure*}[t]
\centering
\includegraphics[width=.95\textwidth,, angle=0]{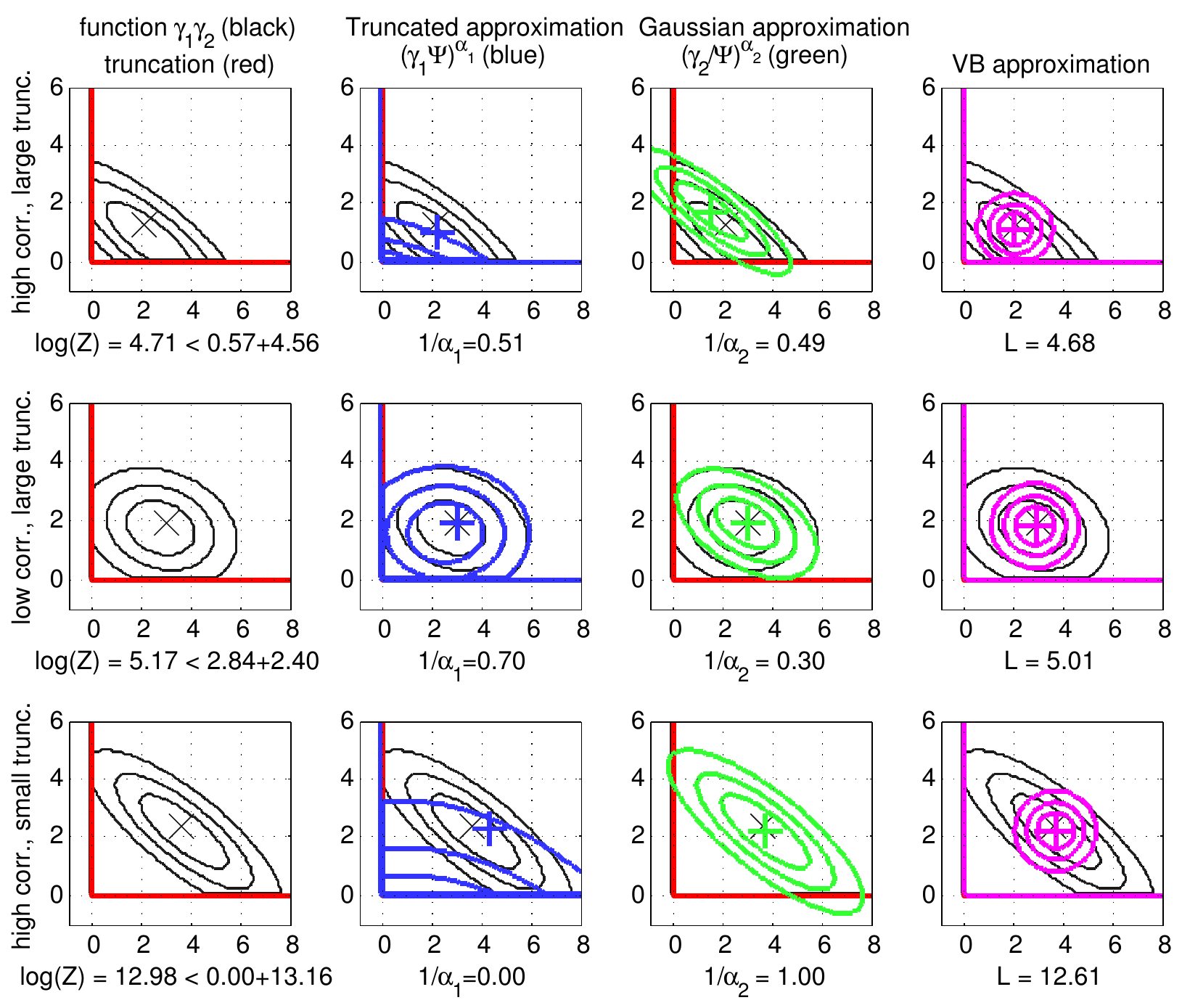}
\vspace{-10pt}
\caption{2D Truncated Gaussian integration. Each row represents a different
correlation/truncation setting. From left to right, the columns show 
1) the
target function $\gamma_1\gamma_2$, 2) its first tractable approximation
$(\gamma_1\Psi)^{\alpha_1}$ (product of orthogonal univariate function), 3)
its second tractable approximation $(\gamma_2/\Psi)^{\alpha_2}$ (correlated
Gaussian distribution) and 4) VB approximation. 
Symbols 'x' and '+' are the exact and approximate means, 
respectively.
}
\label{fig:TruncatedGauss}
\end{figure*}

\subsection{Truncated multi-variate Gaussian integration}
Most of the truncated multivariate Gaussian integration problems with linear
truncations can be put under the canonical form (\ref{univ-gaussian-integral})
where truncations are orthogonal,\footnote{ In their general form, truncated Gaussian integration problems
are based on the estimation of $\int_{\Re^n}  \prod_{i=1}^n f_i(\alpha_i +
\beta_i\transp t_i) e^{-\frac 12 \t\transp \A \t + b\transp \t} d\t$, but a
change of variable $\alpha_i + \beta_i\t_i \rightarrow \t_i$ leads to the
canonical form \eqref{univ-gaussian-integral} if there is no parallel truncation
lines, i.e. box constraints. Box constraints can be handled by a simple
modification involving two-sided univariate truncations. 
}
 i.e.
 $f_i(x)=\Indic{x\ge 0}$ for all $i\in\{1,\cdots,n\}$.
Integrating truncated correlated Gaussian is a known open problem for which
several approximation techniques have been proposed.
In numerical approximation,
adaptive quadratures approach have been well
investigated~\citep{genz2009computation}, but are still limited to small
dimensions. Approximate inference techniques,
such as Expectation-Propagation (EP), have been recently proposed, but the algorithm
remains unstable, even after specific improvements to increase the accuracy of
the method~\citep{Cunningham-et-al-Gaussian-EP2012}.
One of the reason is the fact that EP does not give any guarantee about the
approximation. To be used in a learning framework, upper and lower
bounds to the integral~(\ref{univ-gaussian-integral}) are often very useful. 
We focus here on the upper bound.\footnote{Lower bounding is not straightforward, since the
classical approach to obtain lower bounds is based on the information inequality
requires a class of approximation which is is contained in the support of the
target distribution, and this is not the case for multivariate Gaussian
distributions.} 
\paragraph{Initialization}
We need to initialize parameters so that the integral is tractable. This is not always trivial, but in principle, 
any point in the convex set $\tauspace=\{(\btau_1\transp, \btau_2\transp)\transp \in\Re^{2n} | 0\prec\diag{\btau_1}\prec A \}$ leads to a finite integral. For example, setting $\tau_{1i}$ to half of the minimum Eigen value of $A$ lies within the convex set.

\paragraph{Bound minimization}
After simplification, we get the following overall objective for the upper bound of a 
multivariate truncated Gaussian:
 \begin{flalign}
 \label{eq:obj}  
	\frac 1{\alpha_1}  \sum_i  \log\unigint{f_i}(\tau_{1i},\tau_{2i}, \alpha_1)  
  	 + \frac 1{\alpha_2} J({\alpha_2}(A-\diag{\btau_1}), {\alpha_2}(b-\tau_2))  -\frac n2 \log (2\pi). %
\end{flalign}

Figure~\ref{fig:TruncatedGauss} presents results on a two dimensional truncated gaussian integration problem. The optimal value of $\alpha_1$ depends on both the level of truncation and the correlation. %

\section{Comparison with Variational Bayes}

\subsection{Variational \Holder vs.~Variational Bayes}
One can compare the VH inequality~(\ref{eq:VH-logspace}) to the one provided by the \emph{Variational Bayes} 
(VB) inequality:
\begin{align}
\log\I^*
\ge 
\int \log\gamma_1(\z) dq(\z) + 
\int \log\gamma_2(\z) dq(\z) +
\mathcal{H}(Q)
\enspace.
\label{eq:VB-logspace}
\end{align}
for any distribution $Q$ absolutely continuous with respect to $\nu$, where $\mathcal{H}$ denotes the information entropy and $q=\frac{dQ}{d\nu}$. 

VB provides a lower bound to the log-sum-exp function, while the VH provides an upper bound. One disadvantage
 of the VB bound~(\ref{eq:VB-logspace}) is that it is not concave in general, 
leading to objective functions that are difficult to maximize and a bound that
does not come with theoretical guarantees. Another disadvantage is that the approximating
distribution $Q$ must have a support included in the base distribution $\nu$, which is not 
always convenient when the target distribution has subspaces with zero probability. 
This \emph{zero-avoiding} effect of the VB bound can lead to 
crude approximations of the original integral~\citep{minka2005divergence}.

\subsection{Variational Bayes for truncated multi-variate Gaussian integration}
As a comparison, we consider in this section the VB approach, that gives a lower bound to the likelihood. The key idea to be able to apply VB on this problem is to consider independent truncated Gaussian for the approximation family. We start with \eqref{univ-gaussian-integral}: 
\begin{align}
\log I^*:=\log \int_{\Re^n}  \prod_{i=1}^n f_i(t_i)
e^{-\frac 12 \t\transp \A \t + \bm b\transp \t} d\t  \geq L
\end{align}
where $L$ denotes the negative variational free energy, a.k.a.~the variational lower bound, given by
\begin{align}
L := \sum_{i=1}^n \E{q}{\log f_i(t_i)}
-\frac 12  \trace(\A \E{q}{\t \t\transp}) + \bm b\transp \E{q}{\t}  + \entropy[q]
\end{align}
We choose the variational distribution $q$ to be a product of univariate truncated normal distributions which are truncated at zero. Let $q=\prod_i \truncnormdist(\mu_i, \sigma_i)$. The variational bound is given by
\begin{align}{
L &=  -\frac{1}{2}\trace(A \E{q}{\t\t\transp}) + \bm b\transp \E{q}{\t}  +\frac{n}{2}\log(2\pi e)  +\sum_i\Bigl( \log\bigl(\sigma_i \Phi(\frac{\mu_i}{\sigma_i})\bigr) - \frac{\mu_i}{ 2\sqrt{2\pi} \sigma_i \Phi(\frac{\mu_i}{\sigma_i}) } \exp(-\frac{1}{2}\frac{\mu_i^2}{\sigma_i^2}) \Bigr), \nonumber\\
\E{q}{t_i} &= \mu_i + \frac{\sigma_i}{ \sqrt{2\pi} \Phi(\frac{\mu_i}{\sigma_i})}  \exp(-\frac{1}{2}\frac{\mu_i^2}{\sigma_i^2}),  \nonumber\\
\E{q}{t_i^2} &= \mu_i^2 + \sigma_i^2 + \frac{\mu_i\sigma_i}{ \sqrt{2\pi} \Phi(\frac{\mu_i}{\sigma_i})}  \exp(-\frac{1}{2}\frac{\mu_i^2}{\sigma_i^2}) 
\qquad \textnormal{and}\qquad \E{q}{t_i t_j} = \E{q}{t_i} \E{q}{t_j} \forall i\neq j.
}\end{align}
Minimization of the lower bound could be done iteratively by solving one-dimensional truncated Gaussian fits in a round-robin 
fashion, however, in the experiments below, we computed the gradient of the variational objective and used a 
gradient descent technique to find the optimal variational parameters.

\begin{table}[b]%
\begin{center}
\resizebox{0.5\textwidth}{!}{%
\begin{tabular}{|c|c|c|c|c|c|}
\hline
$\kappa$ & $n$ & Genz & EP & VB & VH \\
\hline
0.1&5&5.1499&5.1327&2.9489&6.4169\\
1&5&0.41768&0.41234&0.10715&0.98725\\\hline
0.1&20&24.7689&24.7702&17.5524&28.2854\\
1&20&1.9203&1.9196&0.97199&2.8037\\\hline
0.1&50&66.2055&66.1991&44.3699&68.971\\
1&50&9.9&9.8999&5.9196&11.2919\\\hline
\end{tabular}
}
\resizebox{0.45\textwidth}{!}{%
\begin{tabular}{|c|c|c|c|c|}
\hline
$\kappa$ & $n$ & VB vs EP & VH vs EP & VB vs VH \\
\hline
0.1&5&1.9286&0.35811&1.7518\\
1&5&0.12856&0.073666&0.19735\\\hline
0.1&20&3.0963&3.3148&2.9741\\
1&20&0.19727&0.074579&0.27093\\\hline
0.1&50&5.5944&0.72705&5.9132\\
1&50&0.51551&0.076875&0.58154\\\hline
\end{tabular}
}
\end{center}
\caption{Comparison of log partition function (left) and error in first moment (right)}
\label{tab:results}
\end{table}

\section{Experiments}  %

To undestand the properties of the VH bound, we compared its properties with existing deterministic integration. integration methods in high dimension. We considered the ground-truth to be the method of Alan Genz~\citep{genz2009computation}
which is based on a sophisticated technique of pseudo-random number generation. The main interest is that it can give an error estimate of the error, so that we can evaluate precisely the validity of various techniques. We used the matlab code provided by the author.
Another efficient technique for integration of truncated Gaussians is based on the use of Expectation-Propagation (EP), as
described by \citet{Cunningham-et-al-Gaussian-EP2012}. In this case as well, a matlab code is provided by the authors. 
Finally, we used implemented the VB version described above to obtain a lower bound to the true integral. 
Both VB and VH objectives were minimized using the L-BFGS algorithm provided by matlab \texttt{fminunc} function.

We used multiple correlations settings, where the precision matrix $A$ was obtained using the following rule:
$A:=\kappa I + v*v^T $, where $v$ is drawn from a $n$-dimensional Gaussian distribution with unit covariance matrix.
We varied the correlation by setting $\kappa\in\{0.1,1\}$ and the dimension by varying $n\in\{5, 20, 50\}$.
We also compared the accuracy of the moment computation, since large gap in the bound does not always imply large difference in the results. The method of Genz did not output moments, so we also compared the accuracy of the moment computation by computing the Euclidean norm of the difference between the three methods: EP, VB and VH for the mean
of the target distribution. %

 Table~\ref{tab:results} gives the results. The first 4 columns compare the integral values. We can see that VB correctly estimates a lower bound to the true integral, and that VG consistently gives an upper bound. EP seems to be generally very accurate, sometimes over-estimating, sometimes under-estimating the exact integral. An interesting phenomenon is that Holder is more accurate than VB in the high correlation setting ($\kappa=0.1$. This is expected since VB is unable to use correlation due to the fact that the 
 approximating family is composed by independent truncated Gaussian variables.

When comparing the moment computation, we see that Holder can give very accurate results, even if the gap in the bound was large. We also see that the higher the dimension, the better VH becomes with respect to VB. We also notice that the high correlation setting, VH and EP are closer to each other, compared to VB suggesting again, that high correlation are well handled by the VH approximation.

\section{Generalization for many factors}
\label{many-factor}
Here, we consider the more general case where the integral to compute is the product of $K$ factors, $K>2$:
$\I^*:=\int \prod_{k'=1}^K f_k(\z) d\basem(\z)$ where $f_k$, $k=1,\cdots, K$ are the individual factors.
We have the following results:
\begin{theorem}
The following inequality:
\begin{eqnarray}
\label{eq:generalization}
I^*
&\le& 
\prod_{k=1}^K \left(
	\int_{\Zspace} {\left(\frac{f_k(\z)}{\Psi_k(\z)}\right)}^{\alpha_k} 
	\prod_{k'=1}^K \Psi_{k'}(\z) d\basem(\z)
	\right)^{\frac{1}{\alpha_k}}
\label{eq:Holder2}
\end{eqnarray}
holds for any $\balpha=(\alpha_1,\cdots,\alpha_K)\in(0,\infty)^K$ such that $\sum_{k=1}^K \frac{1}{\alpha_k}=1$
and any function $\Psi_k:\Zspace\mapsto\Re^+$ in $\mathcal{L}_{\alpha_k}$, $k=1,\cdots, K$.
\end{theorem}
\begin{proof}(sketch)
Similarly to the generalization of \Holder's inequality for a product of functions, we can apply the binary VH bound recursively.
\end{proof}\\
One can verify that we recover the results of Section~2 for $K=2$. 
The VH method can be obtained by parameterizing the pivot functions and minimizing~\eqref{eq:generalization}  with respect
to the pivot functions $\Psi_k$ and $\balpha$.
Tightness and approximation properties studied in Section~3
can also be extended to the case of multiple factors.

\section{Discussion}
We have introduced a new family of variational approximations that are based on the 
minimization of an upper bound to the log-partition function. We demonstrated that
the variational inference problem is convex if the variational function is log-linear, 
which has great practical and theoretical advantages over mean-field/VB approximations,
which is the main approach used today by practitioners. We also provide a novel way to
handle Gaussian integration problems. In fact, we could express probit regression as a special case of this problem, and the extension to other distribution is possible in theory. One of the unique feature of this approach is that
the approximation maintains the heavy tails, but we still
a convex objective. Further experiments will be conducted to evaluate how good this 
approximation behaves in Bayesian posterior estimation. 

The VH framework presented here is very general, and can be applied to many models and
optimized using a large variety of algorithms and speedup tricks, similarly to what happened 
with VB and EP other the last two decades. 

We focused mainly on one type of intractable integrals that is common in machine learning 
problems (GLM or linear models with sparse priors), but the approach is generic and
could potentially be applied in many other settings. One of the main area of application
is the inference in graphical models with discrete variables, on which the TRW sum-product
algorithm has been designed~\citep{wainwright2005new}, as well as several other
algorithms dedicated to discrete graphical models~ \citep{liu11d}. It also provides an upper bound to the log-partition function
and is convex if the tree-weights are known. An alternative proof to the TRW bound based on the \Holder inequality
was given by~\citet{minka2005divergence}, and we conjecture that the TRW bound could be 
expressed as a special case of the proposed approach, for example by assuming that there is 
one factor per possible spanning tree.

\newpage
\bibliography{holder}
\bibliographystyle{plainnat}

\renewcommand{\f}{f}
\renewcommand{\g}{g}
\renewcommand{\p}{p}
\renewcommand{\q}{q}

\newpage
\vspace{-0.1in}
\section*{Appendix}
\vspace{-0.1in}
\subsection*{Pre-requisites}
In the following, the symbols $\f$ and $\g$ represent $\nu-$measurable positive functions, and $\p$ and $\q$ are positive  scalars such that $\frac 1\p+\frac 1\q=1$.

\vspace{-0.1in}
\begin{lemma}  
\label{lem:approx0}
For any $\varepsilon>0$, the inequality $\|\f\g\|_1> (1-\varepsilon) \|\f\|_\p \|\g\|_\q$ implies that:
\begin{align}
\left\|
	\frac{\f^{\p}}{\|\f\|^{\p}_{\p}} 
	- 
	\frac{\f\g}{\|\f\|_\p \|\g\|_\q} 
\right\|_1 \le \sqrt{2\varepsilon}
\quad &  \mathrm{if} \quad \p \le 2, \mathrm{\ and}
\nonumber
\\
\left\|
	\frac{\g^{\q}}{\|\g\|^{\q}_{\q}} 
	- 
	\frac{\f\g}{\|\f\|_\p \|\g\|_\q} 
\right\|_1 \le \sqrt{2\varepsilon} 
\quad  & \mathrm{\ if \ } \p\ge 2
\label{eq:fbigger}
\end{align}
\end{lemma} 

\vspace{-0.1in}

\begin{proof}
Let assume that  $\|\f\|_{\p} \le \|\g\|_{\q}$ and $p\in(1,2]$. 
\begin{eqnarray}
&&
\left\|
	\frac{\f^{\p}}{\|\f\|^{\p}_{\p}} 
	- 
	\frac{\f\g}{\|\f\|_\p \|\g\|_\q} 
\right\|_1
\nonumber\\
 &&= 
\left\|
	\frac{\f^{\frac{\p}{2}}}{\|\f\|^{\frac{\p}{2}}_{\p}} \left(\frac{\f^{\frac{\p}{2}}}{\|\f\|^{\frac{\p}{2}}_{\p}}
	- 
	\frac{\f^{1-\frac{\p}{2}}\g}{\|\f\|^{1-\frac{\p}{2}}_\p \|\g\|_\q} 
	 \right)
\right\|_1 
\nonumber\\
&&\le
\underbrace{
\left\|
	\frac{\f^{\frac{\p}{2}}}{\|\f\|^{\frac{\p}{2}}_{\p}} 
\right\|_2
}_{=1}
\left\|	
	\frac{\f^{\frac{\p}{2}}}{\|\f\|^{\frac{\p}{2}}_{\p}}
	- 
	\frac{\f^{1-\frac{\p}{2}}\g}{\|\f\|^{1-\frac{\p}{2}}_\p \|\g\|_\q} 
\right\|_2 
\label{eq:temp1}
\nonumber
\end{eqnarray}
by Cauchy-Schwartz inequality. We can expand the square of the right-hand term in the product:
\begin{eqnarray}
&&\left\|	
	\frac{\f^{\frac{\p}{2}}}{\|\f\|^{\frac{\p}{2}}_{\p}}
	- 
	\frac{\f^{1-\frac{\p}{2}}\g}{\|\f\|^{1-\frac{\p}{2}}_\p \|\g\|_\q} 
\right\|_2^2 
\nonumber\\
&&= 
\underbrace{
\left\|	
	\frac{\f^{\frac{\p}{2}}}{\|\f\|^{\frac{\p}{2}}_{\p}}
\right\|_2^2
}_{=1}
- 2 
\underbrace{
\frac{\|fg\|_1}{\|\f\|_\p \|\g\|_\q}
}_{\ge 1-\varepsilon}
+ 
\underbrace{
\left\|	
	\frac{\f^{1-\frac{\p}{2}}\g}{\|\f\|^{1-\frac{\p}{2}}_\p \|\g\|_\q} 
\right\|_2^2
}_{:=A}
\nonumber
\\
&&\le A - 1 + 2\varepsilon
\label{eq:temp2}
\enspace,
\end{eqnarray}
where we denote by $A$ the quantity:
\begin{eqnarray}
A := \left\|	
	\frac{\f^{1-\frac{\p}{2}}\g}{\|\f\|^{1-\frac{\p}{2}}_\p \|\g\|_\q} 
\right\|_2^2
=\frac{\left\|	\f^{1-\frac{\p}{2}}\g \right\|_2^2}{\|\f\|^{2-\p}_\p \|\g\|_\q^2} 
=
	\frac{\left\|	\f^{2-\p}\g^2 \right\|_1}{\|\f\|^{2-\p}_\p \|\g\|_\q^2} 
\nonumber
\end{eqnarray}
Assuming $\p\le 2$, we can now bound $A$ by using \Holder's inequality with exponents $p'=\frac{\p}{2-\p}$ and $\q'=\frac{\q}{2}$. One can verify that the pair $(p',q')$ is a valid \Holder's exponent: $\q'\ge 1$, $\p'\ge1$ and $\frac{1}{\p'}+\frac{1}{\q'}=\frac{2-\p}{\p} + \frac{2}{\q} = \frac{2}{\p} +\frac{2}{\q} - 1 = 2-1=1$. We obtain:
\begin{eqnarray}
A &\le&
	\frac{\left\|	\f^{2-\p} \right\|_{\p'} \left\|\g^2 \right\|_{\q'} }
		{\|\f\|^{2-\p}_\p \|\g\|_\q^2} 
	=
	\frac{\left\|	\f \right\|_{\p}^{2-\p} \left\|\g \right\|_{\q}^2 }
		{\|\f\|^{2-\p}_\p \|\g\|_\q^2}  = 1
\nonumber
\end{eqnarray}
This results proves that Equation~\eqref{eq:temp2} is upper bounded by $2\varepsilon$, so the Equation~\eqref{eq:temp1} leads to the following inequality:
\begin{eqnarray}
\left\|
	\frac{\f^{\p}}{\|\f\|^{\p}_{\p}} 
	- 
	\frac{\f\g}{\|\f\|_\p \|\g\|_\q} 
\right\|_1 &\le& 
\sqrt{2\varepsilon}
\label{eq:final}
\end{eqnarray}
Equation~\eqref{eq:fbigger} follows by symmetry. 
\end{proof}

\begin{theorem}  
\label{th:approx0}
For any $\varepsilon>0$, the inequality $\|\f\g\|_1> (1-\varepsilon) \|\f\|_\p \|\g\|_\q$ implies that:
\begin{align}
\left\|
	\frac{\f^{\p}}{\|\f\|^{\p}_{\p}} 
	- 
	\frac{\f\g}{\|\f\g\|_1} 
\right\|_1 \le \sqrt{2\varepsilon} + \varepsilon
\quad &  \mathrm{if} \quad \q\ge\p, \mathrm{\ and}
\nonumber
\\
\left\|
	\frac{\g^{\q}}{\|\g\|^{\q}_{\q}} 
	- 
	\frac{\f\g}{\|\f\g\|_1} 
\right\|_1 \le \sqrt{2\varepsilon} + \varepsilon
\quad  & \mathrm{\ if \ } \quad \p\ge\q
\label{eq:fbigger2}
\end{align}
\end{theorem}

\begin{proof}
\begin{eqnarray}
&&\left\|
	\frac{\f^{\p}}{\|\f\|^{\p}_{\p}} 
	- 
	\frac{\f\g}{\|\f\g\|_1} 
\right\|_1
\nonumber\\
&&
\le 
\underbrace{
\left\|
	\frac{\f^{\p}}{\|\f\|^{\p}_{\p}} 
	- 
	\frac{\f\g}{\|\f\|_\p \|\g\|_\q} 
\right\|_1
}_{\le\sqrt{2\varepsilon}\ \mathrm{by~Lemma~\ref{lem:approx0} }}
+
\underbrace{
\left\|
	\frac{\f\g}{\|\f\|_\p \|\g\|_\q} 
	- 
	\frac{\f\g}{\|\f\g\|_1} 
\right\|_1
}_{\le \varepsilon}
\nonumber
\end{eqnarray}
Equation~\eqref{eq:fbigger2} follows by symmetry.
\end{proof}

\subsection*{Proof of Proposition~\ref{th:approx}}
We are now ready to obtain the proof for the bound approximation property:

\begin{proof}
Apply Theorem~\ref{th:approx0} with $f:=\gamma_1\Psi$, $\g:=\gamma_2/\Psi$, $\p:=\alpha_1$, $\q:=\alpha_2$, $I^*:=\|\f\g\|_1$ and $\bar I_\balpha (\Psi):= \|\gamma_1\Psi\|_\p\|\gamma_2/\Psi\|_\q$.
\end{proof}

\end{document}